\def\year{2021}\relax
\documentclass[letterpaper]{article} 
\usepackage{aaai21}  
\usepackage{times}  
\usepackage{helvet} 
\usepackage{courier}  
\usepackage[hyphens]{url}  
\usepackage{graphicx} 
\usepackage{natbib}  
\usepackage{caption} 
\usepackage{url}            
\usepackage{booktabs}       
\usepackage{amsfonts}       
\usepackage{nicefrac}       
\usepackage{microtype}      
\usepackage{graphicx}
\usepackage{math} 
\usepackage{algorithm}
\usepackage[noend]{algorithmic}
\usepackage{tabularx}
\usepackage{subcaption}
\usepackage[page,header]{appendix}

\newcommand{\A}{\mathcal{A}}
\newtheorem{theorem}{Theorem}
\newtheorem{lemma}[theorem]{Lemma}

\newtheorem{corollary}[theorem]{Corollary}
\newtheorem{definition}{Definition}

\newcommand{\algref}[1]{Alg.~\ref{#1}}

\renewcommand{\S}{\mathcal{S}}

\renewcommand{\E}[2]{\mathbb{E}_{#1}\left[#2\right]}
\newcommand{\interior}[1]{%
  {\kern0pt#1}^{\mathrm{o}}%
}

\newcommand{\pd}[1]{\ensuremath{\pi^\dagger(#1)}}

\usepackage{xcolor}
\usepackage{hhline}
\usepackage{multirow}
\usepackage{enumitem}
\title{The Sample Complexity of Teaching-by-Reinforcement on Q-Learning}
\author {
    Xuezhou Zhang,\textsuperscript{\rm 1}
    Shubham Kumar Bharti, \textsuperscript{\rm 1}
    Yuzhe Ma \textsuperscript{\rm 1} 
    Adish Singla \textsuperscript{\rm 2} 
    Xiaojin Zhu\textsuperscript{\rm 1} \\
}
\affiliations {
    \textsuperscript{\rm 1} UW Madison \\
    \textsuperscript{\rm 2} MPI-SWS \\
}
\begin{document}
	
	\maketitle
	\begin{abstract}
		We study the sample complexity of teaching, termed as ``teaching dimension" (TDim) in the literature, for the \textit{teaching-by-reinforcement} paradigm, where the teacher guides the student through rewards. This is distinct from the \textit{teaching-by-demonstration} paradigm motivated by robotics applications, where the teacher teaches by providing demonstrations of state/action trajectories. The teaching-by-reinforcement paradigm applies to a wider range of real-world settings where a demonstration is inconvenient, but has not been studied systematically. 
		In this paper, we focus on a specific family of reinforcement learning algorithms, Q-learning, and characterize the TDim under different teachers with varying control power over the environment, and present matching optimal teaching algorithms.
		Our TDim results provide the minimum number of samples needed for reinforcement learning, and we discuss their connections to standard PAC-style RL sample complexity and teaching-by-demonstration sample complexity results.
		Our teaching algorithms have the potential to speed up RL agent learning in applications where a helpful teacher is available.
	\end{abstract}
	
	\section{Introduction}
	In recent years, reinforcement learning (RL) has seen applications in a wide variety of domains, such as games~\cite{silver2016mastering,mnih2015human}, robotics control~\cite{kober2013reinforcement,argall2009survey} and healthcare~\cite{komorowski2018artificial,shortreed2011informing}. 
	One of the fundamental questions in RL is to understand the sample complexity of learning, i.e. the amount of training needed for an agent to learn to perform a task. 
	In the most prevalent RL setting, an agent learns through continuous interaction with the environment and learns the optimal policy from natural reward signals. For standard algorithms such as Q-learning, naive interaction with MDP suffers exp complexity~\cite{li2012sample}.
	In contrast, many real-world RL scenarios involve a knowledgable (or even omniscient) teacher who aims at guiding the agent to learn the policy faster. 
	For example, in the educational domain, a human student can be modeled as an RL agent, and a teacher will design a minimal curriculum to convey knowledge (policy) to the student (agent) \cite{chuang2020using}. 
	
	In the context of reinforcement learning, teaching has traditionally been studied extensively under the scheme of \textit{teaching-by-demonstration (TbD)}, where the teacher provides demonstrations of state/action trajectories under a good policy, and the agent aims to mimic the teacher as closely as possible \cite{hussein2017imitation}.
	However, in many applications, it is inconvenient for the teacher to demonstrate because the action space of the teacher is distinct from the action space of the learner. In contrast, it is usually easier for the teacher to \textit{teach by reinforcements (TbR)}, i.e. with rewards and punishments. For example, in dog training, the trainer can't always demonstrate the task to be learned, e.g. fetch the ball with its mouth, but instead would let the dog know whether it performs well by giving treats strategically \cite{chuang2020using}; In personalizing virtual assistants, it's easier for the user to tell the assistant whether it has done a good job than to demonstrate how a task should be performed. Despite its many applications, TbR has not been studied systematically.
	
	In this paper, we close this gap by presenting to our knowledge the first results on TbR. Specifically, we focus on a family of RL algorithms called Q-learning. 
	Our main contributions are: 
	\begin{enumerate}[leftmargin=*, nolistsep]
		\item We formulate the optimal teaching problem in TbR.
		\item We characterize the sample complexity of teaching, termed as "teaching dimension" (TDim), for Q-learning under four different teachers, distinguid by their power (or rather constraints) in constructing a teaching sequence.  See Table~\ref{tab:summary} for a summary of results.
		\item For each teacher level, we design an efficient teaching algorithm which matches the TDim.
		\item We draw connections between our results and classic results on the sample complexity of RL and of TbD.
	\end{enumerate}
	\begin{table*}[ht]
		\caption{Our Main Results on Teaching Dimension of Q-Learning}
		\label{tab:summary}
		\centering
		\begin{tabularx}{1.8\columnwidth}{ c | c | c | c | c }
			\toprule
			\textbf{Teacher} & \textbf{Level 1} & \textbf{Level 2} & \textbf{Level 3}& \textbf{Level 4}\\
			\hline
			\textbf{Constraints} & none & respect agent's $a_t$ & $s_{t+1}: P(s_{t+1}|s_t,a_t)>0$& $s_{t+1}\sim P(\cdot|s_t,a_t)$\\
			\hline
			\textbf{TDim} & $S$ & $S(A-1)$ & $ O\left(SAH\left(\frac{1}{1-\epsilon}\right)^D\right)$ & $O\left(SAH\left(\frac{1}{(1-\epsilon)p_{\min}}\right)^D\right)$\\
			\bottomrule
		\end{tabularx}
	\end{table*}

	\section{Related Work}
	\label{sec:related}
	\paragraph{Classic Machine Teaching}
	Since computational teaching was first proposed in~\cite{shinohara1991teachability,goldman1992complexity}, the teaching dimension has been studied in various learning settings. The vast majority focused on batch supervised learning. See \cite{DBLP:journals/corr/ZhuSingla18} for a recent survey. 
	Of particular interest to us though is teaching online learners such as Online Gradient Descent (OGD)~\cite{liu2017iterative,lessard2018optimal}, active learners~\cite{hanneke2007teaching,peltola2019machine}, and sequential teaching for learners with internal learning state~\cite{hunziker2019teaching,mansouri2019preference,DBLP:conf/nips/ChenSAPY18}. In contrast to OGD where the model update is fully determined given the teacher's data, the RL setting differs in that the teacher may not have full control over the agent's behavior (e.g. action selection) and the environment's evolution (e.g. state transition), making efficient teaching more challenging. Several recent work also study data poisoning attacks against sequential learners~\cite{zhang2019online, ma2019policy, jun2018adversarial,zhang2020adaptive, rakhsha2020policy, ma2018data,wang2018data}. The goal of data poisoning is to force the agent into learning some attacker-specified target policy, which is mathematically similar to teaching. 
	
	\paragraph{Teaching by Demonstration}
	Several recent works studied teaching by demonstrations, particularly focusing on inverse reinforcement learning agents (IRL)~\cite{DBLP:conf/nips/TschiatschekGHD19,DBLP:conf/ijcai/KamalarubanDCS19,brown2019machine,haug2018teaching,cakmak2012algorithmic, walsh2012dynamic}. IRL is a sub-field of RL where the learners aim at recovering the reward function from a set of teacher demonstrations to infer a near-optimal policy. Teaching in IRL boils down to designing the most informative demonstrations to convey a target reward function to the agent. Their main difference to our work lies in the teaching paradigm. IRL belongs to TbD where the teacher can directly demonstrate the desired action in each state. The problem of exploration virtually disappears, because the optimal policy will naturally visit all important states. On the other hand, as we will see next, in the TbR paradigm, the teacher must strategically design the reward signal to \textit{navigate} the learner to each state before it can be taught. In other words, the challenge of exploration remains in reinforcement-based teaching, making it much more challenging than demonstration-based teaching. It is worth mentioning that the NP-hardness in finding the optimal teaching strategy, similar to what we establish in this paper (see Appendix), has also been found under the TbD paradigm \cite{walsh2012dynamic}.

	\paragraph{Empirical Study of Teaching-by-Reinforcement} Empirically, teaching in RL has been studied in various settings, such as reward shaping ~\cite{ng1999policy}, where teacher speeds up learning by designing the reward function, and action advising~\cite{torrey2013teaching,DBLP:conf/ijcai/AmirKKG16}, where the teacher can suggest better actions to the learner during interaction with the environment. Little theoretical understanding is available in how much these frameworks accelerate learning. As we will see later, our teaching framework generalizes both approaches, by defining various levels of teacher's control power, and we provide order-optimal teaching strategies for each setting.

	\section{Problem Definitions}
	\label{sec:problem}
	The machine teaching problem in RL is defined on a system with three entities: 
	the underlying MDP environment,
	the RL agent (student), 
	and the teacher. The teaching process is defined in~\algref{alg:protocol}.
	Whenever the boldface word ``\textbf{may}'' appears in the protocol, it depends on the level of the teacher and will be discussed later. In this paper, we assume that there is a clear separation between a training phase and a test phase, similar to the best policy identification (BPI) framework \cite{fiechter1994efficient} in classic RL. In the training phase, the agent interacts with the MDP for a finite number of episodes and outputs a policy in the end. In the test phase, the output policy is fixed and evaluated. In our teaching framework, the teacher can decide when the training phase terminates, and so teaching is regarded as completed as soon as the target policy is learned. Specifically, in the case of Q-learning, we do not require that the estimated Q function converges to the true Q function w.r.t. the deployed policy, which is similarly not required in the BPI or PAC-RL frameworks, but only require that the deployed policy matches the target policy exactly.
	\begin{algorithm}[ht]
		\caption{Machine Teaching Protocol on Q-learning }\label{alg:protocol}
		\begin{flushleft}
			\textbf{Entities:} MDP environment, learning agent with initial Q-table $Q_0$, teacher with target policy $\pi^\dagger$.\\
		\end{flushleft}
		\begin{algorithmic}[1]
			\WHILE{$\pi_t\neq \pi^\dagger$}
			\STATE MDP draws $s_0 \sim \mu_0$ after each episode reset.  But the teacher \textbf{may} override $s_0$.
			\FOR{$t = 0, \ldots H-1$}
			\STATE The agent picks an action $a_{t} = \pi_{t}(s_{t})$ with its current behavior policy $\pi_{t}$.  But the teacher \textbf{may} override $a_t$ with a teacher-chosen action.  
			\STATE The MDP evolves from $(s_t, a_t)$ to produce immediate reward $r_{t}$ and the next state $s_{t+1}$. But the teacher \textbf{may} override $r_{t}$ or move the system to a different next state $s_{t+1}$.
			\STATE The agent updates $Q_{t+1} = f(Q_t,e_t)$ from experience $e_t = (s_t, a_t, r_t, s_{t+1})$.
			\ENDFOR
			\ENDWHILE
			\STATE Once the agent learns $\pi^\dagger$, the teacher ends the teaching phase, and the learned policy is fixed and deployed.
		\end{algorithmic} 
	\end{algorithm}
	\paragraph{Environment $\mathbf{M}$:} We assume that the environment is an episodic Markov Decision Process (MDP) parameterized by $M = (\mathcal S, \mathcal A, R, P, \mu_0,H)$ where $\S$ is the state space of size $S$, $\mathcal A$ is the action space of size $A$, $R: \mathcal S\times \mathcal A \rightarrow \R$ is the reward function, $P: \mathcal S\times \mathcal A \times \mathcal S \rightarrow \R$ is the transition probability, $\mu_0: \mathcal S\rightarrow \R$ is the initial state distribution, and $H$ is the episode length. Next, we define two quantities of interest of an MDP that we will use in our analysis.
	\begin{definition}
		Let the \textbf{minimum transition probability} $p_{\min}$ of an MDP be defined as
		$
		p_{\min} = \min_{s,s'\in \mathcal S, a\in\mathcal A, P(s'|s,a)>0} P(s'|s,a).
		$
	\end{definition}
	
	\begin{definition}
		Let the \textbf{diameter} $D$ of an MDP be defined as the minimum path length to reach the hardest-to-get-to state in the underlying directed transition graph of the MDP. Specifically,
		\begin{eqnarray}
		D = \max_{s\in S} &&\min_{T,(s_0,a_0,s_1,a_1,...,s_T=s)} T \\
		\mbox{s.t. }&&\mu_0(s_0)>0,
		P(s_{t+1} | s_t,a_t)>0, \forall t\nonumber
		\end{eqnarray}
	\end{definition}
	
	\paragraph{RL agent $L$:} We focus on a family of Q-learning agents $L \in \mathcal L$ with the following properties:
	\begin{enumerate}[leftmargin=*, nolistsep]
		\item \textbf{Behavior policy}: The agent behaves according to the  $\epsilon$-greedy policy for some $\epsilon\in [0,1]$, i.e.
	\begin{equation*}
		\label{eq:behavior}
		\pi_t(s) := \left\{
		\begin{array}{ll}
			a^*\defeq \argmax_a  Q_t(s, a) & \mbox{ w.p. } 1-\epsilon\\
			\mbox{Unif}(\mathcal A\backslash a^*), & \mbox{ w.p. } \epsilon.
		\end{array}
		\right.
	\end{equation*}
		Note this definition is slightly different but equivalent to standard $\epsilon$-greedy exploration, where we merged the probability of choosing $\argmax_a Q_t(s,a)$ in the second branch into the first. This simplifies our notation later.
		\item \textbf{Learning Update}: Given experience $e_t = (s_t,a_t,r_t,s_{t+1})$ at time step $t$, the learning update $Q_{t+1} = f(Q_t,e_t)$ only modifies the $(s_t,a_t)$ entry of the Q-table. Furthermore, the Q-table is ``controllable'': for any $s_t,a_t,s_{t+1}$, there exists a reward $r$ such that the ranking of $a_t$ within $Q_{t+1}(s_t, \cdot)$ can be made first, last or unchanged, respectively.
	\end{enumerate}
	This family includes common Q-learning algorithms such as the standard $\epsilon$-greedy Q-learning, as well as provably efficent variants like UCB-H and UCB-B~\cite{jin2018q}.
	
	\paragraph{Teacher:} In this paper, we study four levels of teachers from the strongest to the weakest:
	\begin{enumerate}[leftmargin=*, nolistsep]
		\item \textbf{Level 1}: The teacher can generate arbitrary transitions $(s_t,r_t,s_{t+1})\in \S \times \R \times \S$, and override the agent chosen action $a_t$.
		None of these needs to obey the MDP (specifically $\mu_0, R, P$).
		\item \textbf{Level 2}: The teacher can still generate arbitrary current state $s_t$, reward $r_t$ and next state $s_{t+1}$, but cannot override the agent's action $a_t$. The agent has ``free will'' in choosing its action.
		\item \textbf{Level 3}: The teacher can still generate arbitrary reward $r_t$ but can only generate MDP-supported initial state and next state, i.e. $\mu_0(s_0)>0$, and $P(s_{t+1}|s_t,a_t)>0$. However, it does not matter what the actual nonzero MDP probabilities are. 
		\item \textbf{Level 4}: The teacher can still generate arbitrary reward $r_t$ but the initial state and next state must be sampled from the MDPs dynamics, i.e. $s_0\sim\mu_0$ and $s_{t+1} \sim P(\cdot|s_t,a_t)$.
	\end{enumerate}
	In all levels, the teacher observes the current Q-table $Q_t$ and knows the learning algorithm $Q_{t+1} = f(Q_t, e_t)$. 
	
	In this work, we are interested in analyzing the \textbf{teaching dimension}, a quantity of interest in the learning theory literature. 
	We define an RL teaching problem instance by the MDP environment $M$, the student $L$ with initial Q-table $Q_0$, and the teacher's target policy $\pi^\dagger$.  
	We remark that the target policy $\pi^\dagger$ need not coincide with the optimal policy $\pi^*$ for $M$.
	In any case, the teacher wants to control the experience sequence so that the student arrives at $\pi^\dagger$ quickly.
	Specifically,
	\begin{definition}
		Given an RL teaching problem instance $(M, L, Q_0, \pi^\dagger)$, the \textbf{minimum expected teaching length} is
		$
		\mathrm{METaL}(M, L, Q_0, \pi^\dagger) = \min_{T,(s_t,a_t,r_t,s_{t+1})_{0:T-1}} \E{}{T}
		\mbox{, s.t. } \pi_T = \pi^\dagger,
		$
		where the expectation is taken over the randomness in the MDP (transition dynamics) and the learner (stochastic behavior policy).
	\end{definition}
	METal depends on nuisance parameters of the RL teaching problem instance.  For example, if $Q_0$ is an initial Q-table that already induces the target policy $\pi^\dagger$, then trivially METal=0.  
	Following the classic definition of teaching dimension for supervised learning, we define TDim by the hardest problem instance in an appropriate family of RL teaching problems:
	
	\begin{definition}
		The \textbf{teaching dimension} of an RL learner $L$ w.r.t. a family of MDPs $\mathcal{M}$ is defined as the worst-case METal:
		$
		TDim = \max_{\pi^\dagger\in \{\pi:\S\rightarrow\A\},Q_0 \in \R^{S\times A}, M\in \mathcal{M}} \mathrm{METaL}(M, L, \pi^\dagger).
		$
	\end{definition}
	
	\section{Teaching without MDP Constraints}
	We start our discussion with the strongest teachers. These teachers have the power of producing arbitrary state transition experiences that do not need to obey the transition dynamics of the underlying MDP. While the assumption on the teaching power may be unrealistic in some cases, the analysis that we present here provides theoretical insights that will facilitate our analysis of the more realistic/less powerful teaching settings in the next section.
	\subsection{Level 1 Teacher}
	The level 1 teacher is the most powerful teacher we consider. In this setting, the teacher can generate arbitrary experience $e_t$. The learner effectively becomes a ``puppet'' learner - one who passively accepts any experiences handed down by the teacher.

	\begin{theorem}\label{thm: lvl 1}
		For a Level 1 Teacher, any learner $L\in \mathcal L$, and an MDP family $\mathcal M$ with $|\mathcal S| = S$ and a finite action space, the teaching dimension is $TDim = S$.
	\end{theorem}
	
	It is useful to illustrate the theorem with the standard $Q$-learning algorithm, which is a member of $\mathcal L$.
	The worst case happens when $\argmax_a Q_0(s, a) \neq \pi^\dagger(s), \forall s$.
	The teacher can simply choose one un-taught $s$ at each step, and construct the experience $(s_t=s, a_t = \pi^\dagger(s), r_t, s_{t+1}=s')$ where $s'$ is another un-taught state (the end case is handled in the algorithm in appendix). 
	Importantly, the teacher chooses 
	$r_t \in \left\{ \frac{\max Q_t(s_t, \cdot) + \theta - (1-\alpha) Q_t(s_t, a_t)}{\alpha} - {\gamma \max Q_t(s', \cdot)} : \theta > 0\right\}$,
	knowing that the standard $Q$-learning update rule $f$ is 
	$Q_{t+1}(s_t, a_t) = (1-\alpha) Q_t(s_t, a_t) + \alpha (r_t + \gamma \max_{a \in A} Q_t(s', a))$.
	This ensures that 
	$Q_{t+1}(s, \pi^\dagger(s)) = \max_{a \neq \pi^\dagger(s)} Q_0(s, a) + \theta > \max_{a \neq \pi^\dagger(s)} Q_0(s, a)$,
	and thus the target policy is realized at state $s$.
	Subsequent teaching steps will not change the action ranking at state $s$. 
	The same teaching principle applies to other learners in $\mathcal L$.

	\subsection{Level 2 Teacher}
	At level 2 the teacher can still generate arbitrary reward $r_t$ and next state $s_{t+1}$, but now it cannot override the action $a_t$ chosen by the learner. 
	This immediately implies that the teacher can no longer teach the desired action $\pi^\dagger(s)$ in a single visit to $s$: for example, $Q_0$ may be such that $Q_0(s, \pi^\dagger(s))$ is ranked last among all actions. 
	If the learner is always greedy with $\epsilon=0$ in~\eqref{eq:behavior}, the teacher will need to visit $s$ for $(A-1)$ times, each time generating a punishing $r_t$ to convince the learner that the top non-target action is worse than $\pi^\dagger(s)$.
	However, for a learner who randomly explores with $\epsilon>0$ it may perform $\pi^\dagger(s)$ just by chance, and the teacher can immediately generate an overwhelmingly large reward to promote this target action to complete teaching at $s$; it is also possible that the learner performs a non-target action that has already been demoted and thus wasting the step.
	Despite the randomness, interestingly our next lemma shows that for any $\epsilon$ it still takes in expectation $A-1$ visits to a state $s$ to teach a desired action in the worst case.
	\begin{lemma}\label{thm:expected}
		For a Level 2 Teacher, any learner in $ \mathcal L$, and an MDP family $\mathcal M$ with action space size $A$, it takes at most $A-1$ visits in expectation to a state $s$ to teach the desired action $\pi^\dagger(s)$ on $s$.
	\end{lemma}
	\textit{Proof Sketch}:
	Let us consider teaching the target action $\pi^\dagger(s)$ for a particular state $s$.  Consider a general case where there are $A-c$ actions above $\pi^\dagger(s)$ in the current ordering $Q_t(s,\cdot)$.
	In the worst case $c=1$.
	We define the function $T(x)$ as the expected number of visits to $s$ to teach the target action $\pi^\dagger(s)$ to the learner when there are $x$ higher-ranked actions. 
	For any learner in $\mathcal L$, the teacher can always provide a suitable reward to either move the action selected by the learner to the top of the ordering or the bottom. 
	Using dynamic programming we can recursively express $T(A-c)$ as 
	\begin{align*}
		T(A-c) &= 1 + (c-1)\frac{\epsilon}{A-1}T(A-c)+\\
		&(1-\epsilon +  (A-c-1)\frac{\epsilon}{A-1})T({A-c-1}).
	\end{align*}
	Solving it gives
	$
	T(A-c) = \frac{A-c}{(1-(c-1)\frac{\epsilon}{A-1})}$, which implies
	$\max_c T(A-c) = T(A-1) = A-1.
	$\qed
	Lemma \ref{thm:expected} suggests that the agent now needs to visit each state at most $(A-1)$ times to learn the target action, and thus teaching the target action on all states needs at most $S(A-1)$ steps:
	\begin{theorem}\label{thm: lvl 2}
		For a Level 2 Teacher, any learner in $ \mathcal L$, and an MDP family $\mathcal M$ with state space size $S$ and action space size $A$, the teaching dimension is $TDim = S(A-1)$.
	\end{theorem}
	We present a concrete level-2 teaching algorithm in the appendix.
	For both Level 1 and Level 2 teachers, we can calculate the exact teaching dimension due to a lack of constraints from the MDP. 
	The next levels are more challenging, and we will be content with big O notation.
	\begin{figure*}[ht]
		\begin{center}
			\includegraphics[width=0.6\textwidth]{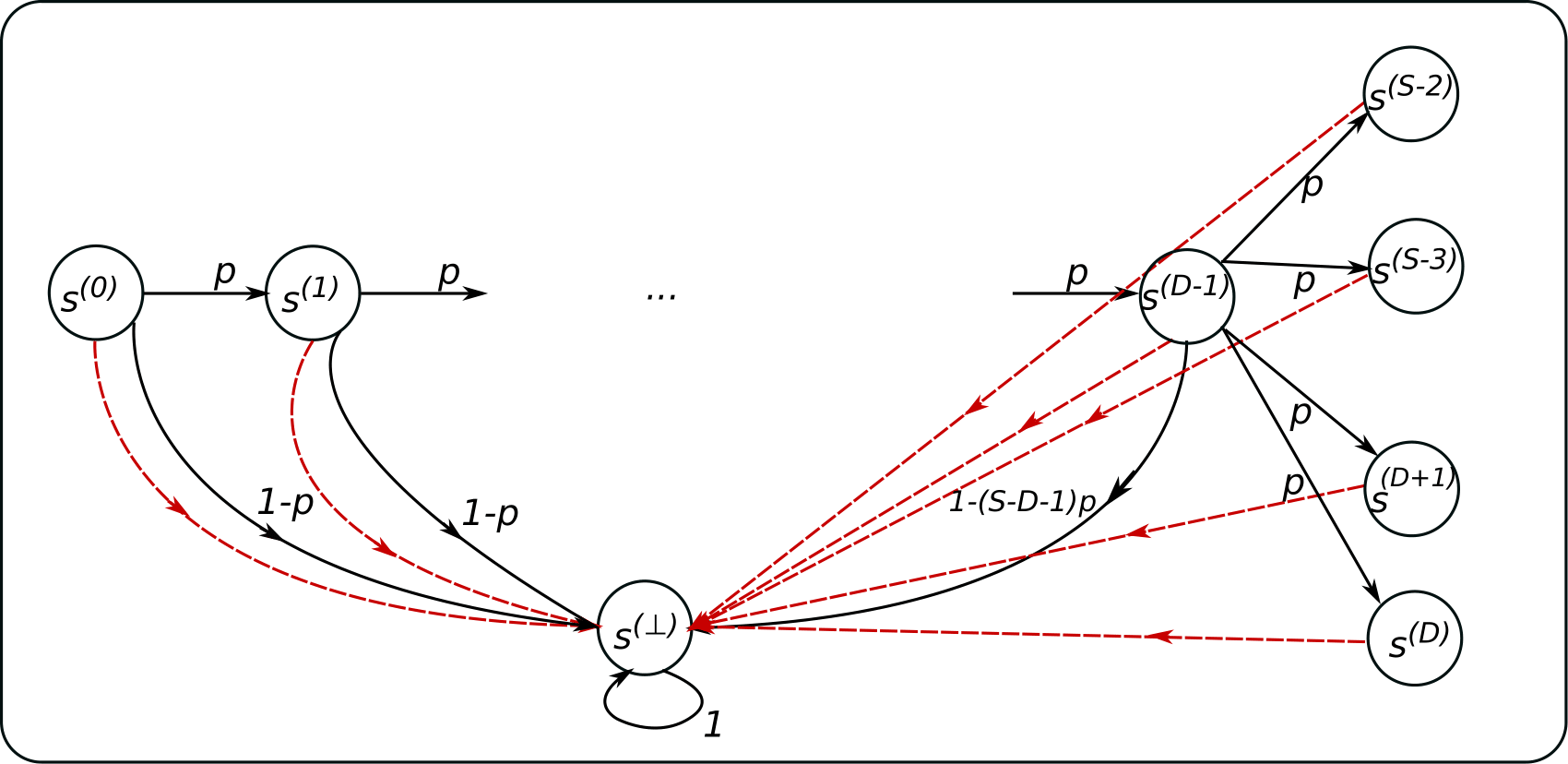}
			\caption{The ``peacock" MDP}
			\label{fig:peacock}
		\end{center}
	\end{figure*}
	\section{Teaching subject to MDP Constraints}
	In this section, we study the TDim of RL under the more realistic setting where the teacher must obey some notion of MDP transitions. In practice, such constraints may be unavoidable. For example, if the transition dynamics represent physical rules in the real world, the teacher may be physically unable to generate arbitrary $s_{t+1}$ given $s_t, a_t$ (e.g. cannot teleport).
	
	\subsection{Level 3 Teacher}
	In Level 3, the teacher can only generate a state transition to $s_{t+1}$ which is in the support of the appropriate MDP transition probability, i.e. $s_{t+1}\in \{s: P(s \mid s_t,a_t)>0\}$. 
	However, the teacher can freely choose $s_{t+1}$ within this set regardless of how small $P(s_{t+1} \mid s_t,a_t)$ is, as long as it is nonzero.
	Different from the previous result for Level 1 and Level 2 teacher, in this case, we are no longer able to compute the exact TDim of RL. Instead, we provide matching lower and upper-bounds on TDim.
	\begin{theorem}\label{thm:l3_lower}
		For Level 3 Teacher, any learner in $\mathcal L$ with $\epsilon$ probability of choosing non-greedy actions at random, an MDP family $\mathcal M$ with episode length $H$ and diameter $D\leq H$, the teaching dimension is lower-bounded by
		\begin{equation}
		TDim \geq \Omega\left((S-D)AH\left(\frac{1}{1-\epsilon}\right)^D\right).
		\end{equation}
	\end{theorem}
	\textit{proof.}
	The proof uses a particularly hard RL teaching problem instance called the ``peacock MDP'' in Figure~\ref{fig:peacock} to produce a tight lower bound.
	The MDP has $S$ states where the first $D$ states form a linear chain (the ``neck''),
	the next $S-D-1$ states form a star (the ``tail''), and the last state $s^{(\bot)}$ is a special absorbing state. 
	The absorbing state can only be escaped when the agent resets after episode length $H$.
	The agent starts at $s^{(0)}$ after reset.  
	It is easy to verify that the peacock MDP has a diameter $D$.
	Each state has $A$ actions.
	For states along the neck, the $a_1$ action (in black) has probability $p>0$ of moving right, and probability $1-p$ to go to the absorbing state $s^{(\bot)}$; all other actions (in red) have probability $1$ of going to $s^{(\bot)}$. 
	The $a_1$ action of $s^{(D-1)}$ has probability $p$ to transit to each of the tail states. In the tail states, however, all actions lead to the absorbing state with probability $1$.
	We consider a target policy $\pi^\dagger$ where $\pi^\dagger(s)$ is a red action $a_2$ for all the tail states $s$.
	It does not matter what $\pi^\dagger$ specifies on other states.
	We define $Q_0$ such that $a_2$ is $\argmin_a Q_0(s,a)$ for all the tail states.
	
	The proof idea has three steps:
	(1) By Lemma~\ref{thm:expected} the agent must visit each tail node $s$ for $A-1$ times to teach the target action $a_2$, which was initially at the bottom of ${Q_0}(s, \cdot)$. 
	(2) But the only way that the agent can visit a tail state $s$ is to traverse the neck every time.
	(3) The neck is difficult to traverse as any $\epsilon$-exploration sends the agent to $s^{(\bot)}$ where it has to wait for the episode to end.
	
	We show that the expected number of steps to traverse the neck once is $H(\frac{1}{1-\epsilon})^D$ even in the best case, where the agent's behavior policy~\eqref{eq:behavior} prefers $a_1$ at all neck states.
	In this best case, the agent will choose $a_1$ with probability $1-\epsilon$ at each neck state $s$.
	If $a_1$ is indeed chosen by the agent, by construction the support of MDP transition $P(\cdot \mid s, a_1)$ contains the state to the right of $s$ or the desired tail state (via the transition with probability $p>0$).
	This enables the level 3 teacher to generate such a transition regardless of how small $p$ is (which is irrelevant to a level 3 teacher).
	In other words, in the best case, the agent can move to the right once with probability $1-\epsilon$.
	A successful traversal requires moving right $D$ times consecutively, which has probability $(1-\epsilon)^D$.
	The expected number of trials (to traverse) until success is $(\frac{1}{1-\epsilon})^D$.
	A trial fails if any time during a traversal the agent picked an exploration action $a$ other than $a_1$.
	Then the support of $P(\cdot \mid s, a)$ only contains the absorbing state $s^{(\bot)}$, so the teacher has no choice but to send the agent to $s^{(\bot)}$.
	There the agent must wait for the episode to complete until resetting back to $s^{(0)}$.
	Therefore, any failed trial incurs exactly $H$ steps of wasted teaching.
	Putting things together, the expected number of teaching steps until a successful neck traversal is done is at least $H(\frac{1}{1-\epsilon})^D$.
	
	There are $S-D-1$ tail states.  Each needs an expected $A-1$ neck traversals to teach.  This leads to the lower bound $(S-D-1)(A-1)H(\frac{1}{1-\epsilon})^D = \Omega\left((S-D)AH\left(\frac{1}{1-\epsilon}\right)^D\right)$. \qed
	
	Our next result shows that this lower bound is nearly tight, by constructing a level-3 teaching algorithm that can teach any MDP with almost the same sample complexity as above.
	\begin{theorem}\label{thm:l3_upper}
		Under the same conditions of Theorem~\ref{thm:l3_lower}, the level-3 teaching dimension is upper-bounded by
		\begin{equation}
		TDim \leq O\left(SAH\left(\frac{1}{1-\epsilon}\right)^D\right).
		\end{equation}
	\end{theorem}
	
	\textit{proof.}
	We analyze a level-3 teaching algorithm NavTeach (Navigation-then-Teach) which, like any teaching algorithm, provides an upper bound on TDim.
	The complete NavTeach algorithm is given in the appendix; we walk through the main steps on an example MDP in Figure~\ref{fig:navteach}(a).  For the clarity of illustration the example MDP has only two actions $a_1, a_2$ and deterministic transitions (black and red for the two actions respectively), though NavTeach can handle fully general MDPs. The initial state is $s^{(0)}$.

	Let us say NavTeach needs to teach the ``always take action $a_1$'' target policy: $\forall s, \pi^\dagger(s)=a_1$.
	In our example, these black transition edges happen to form a tour over all states, but the path length is 3 while one can verify the diameter of the MDP is only $D=2$.
	In general, though, a target policy $\pi^\dagger$ will not be a tour.  It can be impossible or inefficient for the teacher to directly teach $\pi^\dagger$.
	Instead, NavTeach splits the teaching of $\pi^\dagger$ into subtasks for one ``target state'' $s$ at a time over the state space in a carefully chosen order.
	Importantly, before teaching each $\pi^\dagger(s)$ NavTeach will teach a different navigation policy $\pi^{nav}$ for that $s$.
	The navigation policy $\pi^{nav}$ is a partial policy that creates a directed path from $s^{(0)}$ to $s$, which is similar to the neck in the earlier peacock example.
	The goal of $\pi^{nav}$ is to quickly bring the agent to $s$ often enough so that the \emph{target} policy $\pi^\dagger(s)=a_1$ can be taught at $s$.  That completes the subtask at $s$.
	Critically, NavTeach can maintain this target policy at $s$ forever, while moving on to teach the next target state $s'$.
	This is nontrivial because NavTeach may need to establish a different navigation policy for $s'$: the old navigation policy may be partially reused, or demolished.
	Furthermore, all these need to be done in a small number of steps.
	\begin{figure*}[ht]
		\begin{subfigure}{.25\textwidth}
			\centering
			\includegraphics[width=0.85\textwidth]{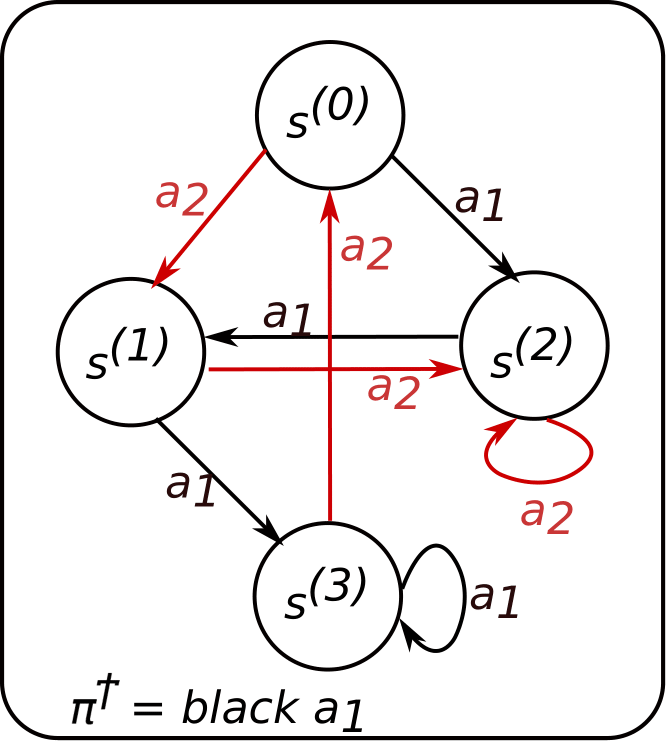} 
			\caption{MDP}
			\label{fig:sub-first}
		\end{subfigure}
		\begin{subfigure}{.25\textwidth}
			\centering
			\includegraphics[width=0.95\textwidth]{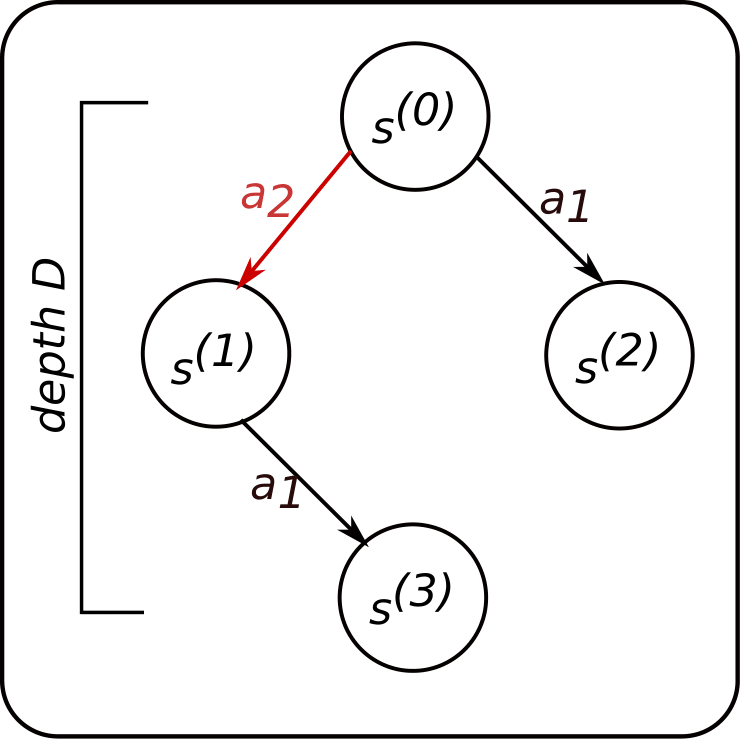}
			\caption{Breadth-First Tree}
			\label{fig:sub-second}
		\end{subfigure}
		\begin{subfigure}{.25\textwidth}
			\centering
			\includegraphics[width=0.925\textwidth]{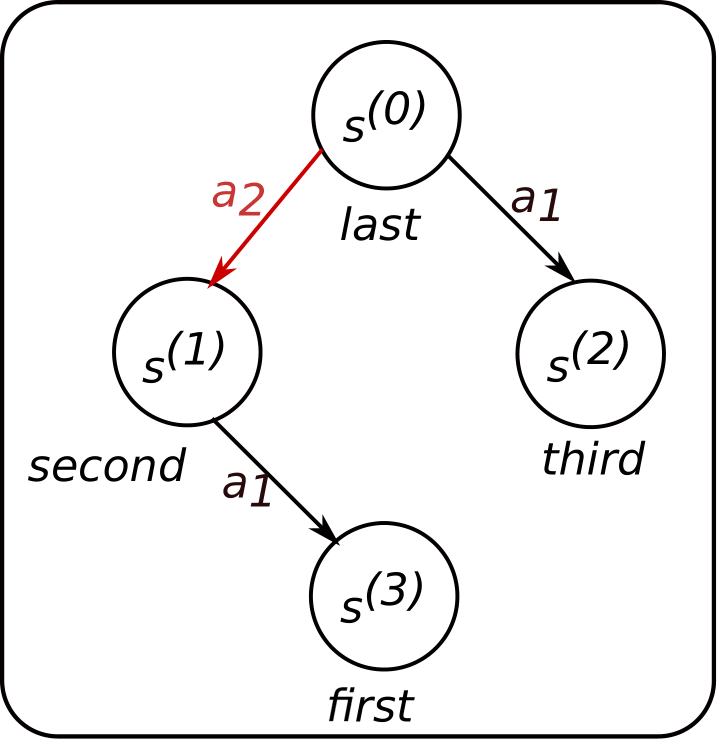}
			\caption{Depth-First Traversal}
			\label{fig:sub-third}
		\end{subfigure}
		\begin{subfigure}{.2\textwidth}
			\centering
			\includegraphics[width=0.95\textwidth]{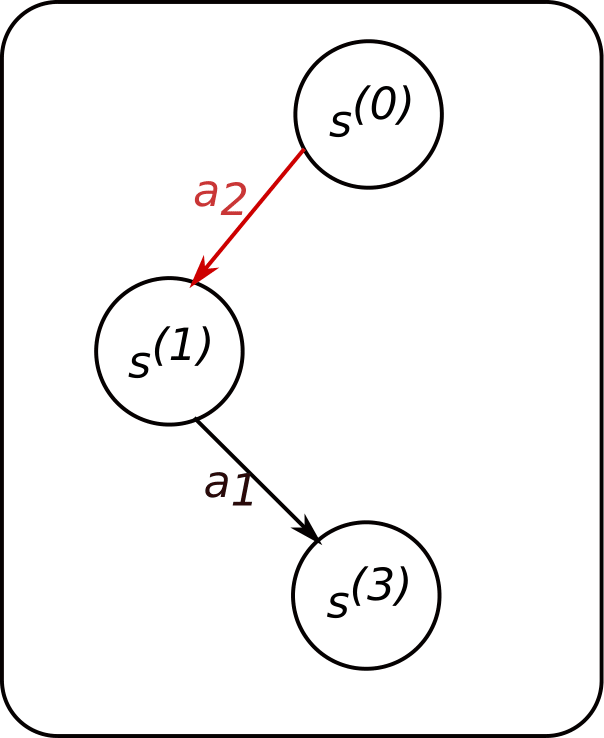}
			\caption{Navigation Policy}
			\label{fig:sub-fourth}
		\end{subfigure}
		\caption{NavTeach algorithm demo}
		\label{fig:navteach}
	\end{figure*}
	We now go through NavTeach on Figure~\ref{fig:navteach}(a).
	The first thing NavTeach does is to carefully plan the subtasks.
	The key is to make sure that (i) each navigation path is at most $D$ long; (ii) once a target state $s$ has been taught: $\pi^\dagger(s)=a_1$, it does not interfere with later navigation.
	To do so, NavTeach first constructs a directed graph where the vertices are the MDP states, and the edges are non-zero probability transitions of all actions.
	This is the directed graph of Figure~\ref{fig:navteach}(a), disregarding color.
	NavTeach then constructs a breadth-first-tree over the graph, rooted at $s^{(0)}$.  This is shown in Figure~\ref{fig:navteach}(b).
	Breadth-first search ensures that all states are at most depth $D$ away from the root.
	Note that this tree may uses edges that correspond to non-target actions, for example the red $a_2$ edge from $s^{(0)}$ to $s^{(1)}$.
	The ancestral paths from the root in the tree will form the navigation policy $\pi^{nav}$ for each corresponding node $s$.
	Next, NavTeach orders the states to form subtasks.
	This is done with a depth-first traversal on the tree: a depth-first search is performed, and the nodes are ranked by the last time they are visited.  This produces the order in Figure~\ref{fig:navteach}(c).
	The order ensures that later navigation is ``above'' any nodes on which we already taught the target policy, thus avoiding interference.
	
	Now NavTeach starts the first subtask of teaching $\pi^\dagger(s^{(3)})=a_1$, i.e. the black self-loop at $s^{(3)}$.
	As mentioned before, NavTech begins by teaching the navigation policy $\pi^{nav}$ for this subtask, which is the ancestral path of $s^{(3)}$ shown in Figure~\ref{fig:navteach}(d).
	How many teaching steps does it take to establish this $\pi^{nav}$?  Let us look at the nodes along the ancestral path.
	By Lemma~\ref{thm:expected} the agent needs to be at the root $s^{(0)}$ $A-1$ times in expectation in order for the teacher to teach $\pi^{nav}(s^{(0)})=a_2$; this is under the worst case scenario where the initial agent state $Q_0$ places $a_2$ at the bottom in state $s^{(0)}$.
	We will assume that after a visit to $s^{(0)}$, the remaining episode is simply wasted.
	\footnote{It is important to note that the teacher always has a choice of $r_t$ so that the teaching experience does not change the agent's $Q_t$ state.  For example, if the agent's learning algorithm $f$ is a standard Q-update, then there is an $r_t$ that keeps the Q-table unchanged.  So while in wasted steps the agent may be traversing the MDP randomly, the teacher can make these steps ``no-op'' to ensure that they do not damage any already taught subtasks or the current navigation policy.}
	Therefore it takes at most $H(A-1)$ teaching steps to establish $\pi^{nav}(s^{(0)})=a_2$.
	After that, it takes at most $H(A-1)(\frac{1}{1-\epsilon})$ expected number of teaching steps to teach $\pi^{nav}(s^{(1)})=a_1$.
	This is the same argument we used in Theorem~\ref{thm:l3_lower}: the teacher needs to make the agent traverse the partially-constructed ancestral path (``neck'') to arrive at $s^{(1)}$.
	The worst case is if the agent performs a random exploration action anywhere along the neck; it falls off the neck and wastes the full episode.
	In general to establish a nagivation policy $\pi^{nav}$ with path length $d$, NavTeach needs to teach each navigation edge at depth $i=1 \ldots d$ with  at most 
	$H(A-1)(\frac{1}{1-\epsilon})^{i-1}$
	teaching steps, respectively.
	After establishing this $\pi^{nav}$ for $s^{(3)}$, NavTeach needs to go down the neck frequently 
	to ensure that it visits $s^{(3)}$ $(A-1)$ times and actually teach the target policy $\pi^\dagger(s^{(3)})=a_1$.
	This takes an additional at most $H(A-1)(\frac{1}{1-\epsilon})^{d}$ teaching steps.

	When the $s^{(3)}$ subtask is done, according to our ordering in Figure~\ref{fig:navteach}(c) NavTeach will tackle the subtask of teaching $\pi^\dagger$ at $s^{(1)}$.  Our example is lucky because this new subtask is already done as part of the previous navigation policy.
	The third subtask is for $s^{(2)}$, where NavTeach will have to establish a new navigation policy, namely $\pi^{nav}(s^{(0)})=a_1$.  And so on.
	How many total teaching steps are needed?
	\textbf{A key insight is NavTeach only needs to teach any navigation edge in the breadth-first tree exactly once.}
	This is a direct consequence of the depth-first ordering: there can be a lot of sharing among navigation policies; a new navigation policy can often re-use most of the ancestral path from the previous navigation policy.
	Because there are exactly $S-1$ edges in the breadth-first tree of $S$ nodes, the total teaching steps spent on building navigation policies is the sum of $S-1$ terms of the form $H(A-1)(\frac{1}{1-\epsilon})^{i-1}$ where $i$ is the depth of those navigation edges.  
	We can upperbound the sum simply as $(S-1)H(A-1)(\frac{1}{1-\epsilon})^{D}$.
	On the other hand, the total teaching steps spent on building the target policy $\pi^\dagger$ at all target states is the sum of $S$ terms of the form $H(A-1)(\frac{1}{1-\epsilon})^{d}$ where $d$ is the depth of the target state.
	We can upperbound the sum similarly as $SH(A-1)(\frac{1}{1-\epsilon})^{D}$.
	Putting navigation teaching and target policy teaching together, we need at most $(2S-1)H(A-1)(\frac{1}{1-\epsilon})^{D} = O\left(SAH\left(\frac{1}{1-\epsilon}\right)^D\right)$ teaching steps. \qed
	
	We remark that more careful analysis can in fact provide matching lower and upper bounds up to a constant factor, in the form of $\Theta\left((S-D)AH(1-\epsilon)^{-D} + H\frac{1-\epsilon}{\epsilon}[(1-\epsilon)^{-D}-1]\right)$. We omit this analysis for the sake of a cleaner presentation. However, the matching bounds imply that a deterministic learner, with $\epsilon = 0$ in the $\epsilon$-greedy behavior policy, has the smallest teaching dimension. This observation aligns with the common knowledge in the standard RL setting that algorithms exploring with stochastic behavior policies are provably sample-inefficient \cite{li2012sample}.
	\begin{corollary}\label{thm: cor}
		For Level 3 Teacher, any learner in $\mathcal L$ with $\epsilon=0$, and any MDP $M$ within the MDP family $\mathcal M$ with $|\mathcal S| = S$, $|\mathcal A| = A$, episode length $H$ and diameter $D\leq H$, we have
		$
		TDim= \Theta\left(SAH\right).
		$
	\end{corollary}

	\subsection{Level 4 Teacher}
	In Level 4, the teacher no longer has control over state transitions. The next state will be sampled according to the transition dynamics of the underlying MDP, i.e. $s_{t+1}\sim P(\cdot|s_t,a_t)$. As a result, the only control power left for the teacher is the control of reward, coinciding with the reward shaping framework. Therefore, our results below can be viewed as a sample complexity analysis of RL under \textit{optimal reward shaping}. Similar to Level 3, we provide near-matching lower and upper-bounds on TDim. 
	\begin{theorem}\label{thm: l4_lower}
		For Level 4 Teacher, and any learner in $ \mathcal L$, and an MDP family $\mathcal M$ with $|\mathcal S| = S$, $|\mathcal A| = A\geq 2$, episode length $H$, diameter $D\leq H$ and minimum transition probability $p_{\min}$, the teaching dimension is lower-bounded by
		$
		TDim \geq \Omega\left((S-D)AH\left(\frac{1}{p_{\min}(1-\epsilon)}\right)^D\right).
		$
	\end{theorem}
	\begin{theorem}\label{thm: l4_upper}
		For Level 4 Teacher, any learner in $ \mathcal L$, and any MDP $M$ within the MDP family $\mathcal M$ with $|\mathcal S| = S$, $|\mathcal A| = A$, episode length $H$, diameter $D\leq H$ and minimum transition probability $p_{\min}$, the Nav-Teach algorithm in the appendix can teach any target policy $\pi^\dagger$ in a expected number of steps at most
		$
		TDim \leq O\left(SAH\left(\frac{1}{p_{\min}(1-\epsilon)}\right)^D\right).
		$
	\end{theorem}
	The proofs for Theorem \ref{thm: l4_lower} and \ref{thm: l4_upper} are similar to those for Theorem \ref{thm:l3_lower} and \ref{thm:l3_upper}, with the only difference that under a level 4 teacher the expected time to traverse a length $D$ path is at most $H(1/p_{\min}(1-\epsilon))^{D}$ in the worst case.
	The $p_{\min}$ factor accounts for sampling from $P(\cdot \mid s_t, a_t)$.
	Similar to Level 3 teaching, we observe that a deterministic learner incurs the smallest TDim, but due to the stochastic transition, an exponential dependency on $D$ is unavoidable in the worst case.
	\begin{corollary}
		For Level 4 Teacher, any learner in $\A$ with $\epsilon=0$, and any MDP $M$ within the MDP family $\mathcal M$ with $|\mathcal S| = S$, $|\mathcal A| = A$, episode length $H$, diameter $D\leq H$ and minimum transition probability $p_{\min}$, we have
		$
		TDim \leq O\left(SAH\left(\frac{1}{p_{\min}}\right)^D\right).
		$
	\end{corollary}
	
	\section{Sample efficiencies of standard RL, TbD and TbR}
	In the standard RL setting, some learners in the learner family $\mathcal L$, such as UCB-B, are provably efficient and can learn a $\delta$-optimal policy in $O(H^3SA/\delta^2)$ iterations \cite{jin2018q}, where $\delta$-optimal means that the cumulative rewards achieved by the output policy is only $\delta$-worse than the optimal policy, i.e. $V^*(\mu_0)-V^\pi(\mu_0)\leq \delta$. One direct implication of such a measure is that the remote states that are unreachable also hardly affect the policy's performance, so quantities like the diameter of the MDP does not appear in the bound.
	
	In contrast, in our TbR work, we aim at learning the \textit{exact} optimal policy, and will thus suffer exponentially if some states are nearly unreachable. However, if we assume that all states have reasonable visiting probabilities, then even the weakest teacher (Level 3 and 4) can teach the optimal policy in $O(HSA)$ iterations, which is of $H^2$ factor better than the best achievable rate without a teacher. More interestingly, even the learners with a not as good learning algorithm, e.g. standard greedy Q-learning, which can never learn the optimal policy on their own, can now learn just as efficiently under the guidance of an optimal teacher.
	
	Teaching-by-demonstration is the most sample efficient paradigm among the three, because the teacher can directly demonstrate the optimal behavior $\pi^\dagger(s)$ on any state $s$, and effectively eliminate the need for exploration and navigation. If the teacher can generate arbitrary $(s, a)$ pairs, then he can teach any target policy with only $S$ iterations, similar to our Level 1 teacher. If he is also constrained to obey the MDP, then it has been shown that he can teach a $\delta$-optimal policy in $O(SH^2/\delta)$ iterations \cite{sun2017deeply, rajaraman2020toward}, which completely drops the dependency on the action space size $A$ compared to both RL and TbR paradigms. Intuitively, this is due to the teacher being able to directly demonstrate the optimal action, whereas, in both RL and TbR paradigms, the learner must try all actions before knowing which one is better.
	
	In summary, in terms of sample complexity, we have
	\begin{eqnarray}
	RL>TbR>TbD.
	\end{eqnarray}
	
	\section{Conclusion and Discussions}
	We studied the problem of teaching Q-learning agents under various levels of teaching power in the Teaching-by-Reinforcement paradigm. At each level, we provided near-matching upper and lower bounds on the teaching dimension and designed efficient teaching algorithms whose sample complexity matches the teaching dimension in the worst case.
	Our analysis provided some insights and possible directions for future work:
	\begin{enumerate}[leftmargin=*, nolistsep]
		\item \textbf{Agents are hard to teach if they randomly explore:} Even under an optimal teacher, learners with stochastic behavior policies ($\epsilon>0$) necessarily suffer from exponential sample complexity, coinciding with the observation made in the standard RL setting \cite{li2012sample}.
		\item \textbf{Finding METaL is NP-hard:} While we can quantify the worst-case TDim, for a particular RL teaching problem instance we show that computing its METaL is NP-hard in Appendix.
		\item \textbf{The controllability issue:} What if the teacher cannot fully control action ranking in agent's $Q_t$ via reward $r$ (see agent ``Learning Update'' in section~\ref{sec:problem})?  This may be the case when e.g. the teacher can only give rewards in $[0,1]$.
		The TDim is much more involved because the teacher cannot always change the learner's policy in one step. Such analysis is left for future work.
		\item \textbf{Teaching RL agents that are not Q-learners}: In the appendix, we show that our results also generalize to other forms of Temporal Difference (TD) learners, such as SARSA. Nevertheless, it remains an open question of whether even broader forms of RL agents (e.g. policy gradient and actor-critic methods) enjoy similar teaching dimension results.
	\end{enumerate}
	
	\bibliography{TDRL}

 	\onecolumn
	\newpage
	\appendix
	\appendixpage

	\section{The Computational Complexity of Finding METaL}\label{sec:complexity}
	
	In this section, we discuss another aspect of teaching, namely the computational complexity of finding the exact minimum expected teaching length of a particular teaching problem instance, i.e. $METal(M,L,Q_0,\pi^\dagger)$.
	Note this differs from TDim in that it is instance-specific.
	
	For Level 1 and Level 2 teachers, the exact METaL can be found with polynomial-time algorithms Alg.~\ref{alg:puppet} and Alg.~\ref{alg:level2}. 
	Now, we show that for the less powerful Level 3 teacher, finding METaL of a particular instance is NP-hard. In particular, it is as hard as the Asymmetric TSP problem. 
	\begin{definition}
		An Asymmetric TSP problem~\cite{DBLP:journals/corr/abs-1708-04215}, characterized by a directed graph $G = (V,E)$ and a starting vertex $v\in V$, is defined as finding the minimum length path that starts from $v$ and visits all vertices $v'\in V$ at least once.
	\end{definition}

	\begin{theorem}
		Finding the METaL of a Level 3 teaching problem instance is at least as hard as the Asymmetric Traveling Salesman Problem(ATSP), which is NP-hard; This also means that the best polynomial-time approximation algorithm can only achieve a constant-factor approximation.
	\end{theorem}
	
	\begin{proof}
		We show a polynomial-time reduction from ATSP problem to a Level 3 METaL problem. Specifically, we show that for every ATSP problem instance $G=(V,E)$, there exists a Level 3 METaL problem instance $(M,L,Q_0,\pi^\dagger)$ such that the ATSP problem instance has a solution $l$ if and only if the corresponding METaL instance has a solution $l$.
		
		
		The reduction is as follows. Given an ATSP problem instance $\{\text{Graph } G=(V,E), \text{ start vertex = } s_0\}$, we provide a construction to a level 3 METal problem instance $(M,L,Q_0,\pi^\dagger)$. We start by constructing the MDP first. 
		The vertex set $V$ forms the state space of the MDP. 
		Each state $s$ has exactly two actions $a^{(0)}$ and $a^{(1)}$.
		The support of the transition probability distributions $P(\cdot\mid s, a^{(0)})$ and $P(\cdot\mid s, a^{(1)})$ are the same: they are the outgoing edges of $s$ in the graph $G$. 
		The exact value of these probabilities and the reward function does not matter, since a level 3 teacher has the power to override them. The initial state distribution $\mu_0$ is concentrated on $s_0$.
		We construct a $Q_0$ that favors action $a^{(0)}$ in each state, and the target policy $\pi^\dagger(s) = a^{(1)}$ for each state $s \in \S$. The horizon is $H = D^2$ , where $D$ is the diameter of the graph $G$. The learner is in $\mathcal L$.
		
		\textit{Claim 1}: If an ATSP problem instance $\{G=(V,E), s_0\}$ has a solution $l$, then the level 3 METaL problem instance $(M,L,Q_0,\pi^\dagger)$ has a solution $l$. 
		
		To verify Claim 1, note	the teacher needs to make the learner visit every state exactly once to teach the target action $a^{(1)}$ in that state. This is because initially every state is untaught (by construction $Q_0$ prefers $a^{(0)})$. Further, each state $s$ has exactly two actions and no matter which action the learner takes, the teacher can provide a suitable reward to push the target action $a^{(1)}$ to the top of Q-value ordering.
		If the ATSP problem has a solution $s_{i_0} = s_0 \rightarrow s_{i_1} \rightarrow \cdots s_{i_{l-1}}$, it is possible for the teacher to provide the state transitions $s_{i_0} = s_0 \rightarrow s_{i_1} \rightarrow \cdots s_{i_{l-1}}$ that visits all the states in the least number of time steps and thus teach the target policy optimally. This is because for every edge $s_i \rightarrow s_j$ in the graph, the transition $P(\cdot\mid s_i, a)$ supports $s_j$ for both the actions.
		
		\textit{Claim 2}: If the level 3 METaL problem instance $(M,L,Q_0,\pi^\dagger)$ has a solution $l$, then the ATSP problem instance $\{G=(V,E), s_0\}$  has a solution $l$. 
		
		We prove this by contradiction. Let say the METal problem instance $(M,L,Q_0,\pi^\dagger)$ has a solution $l$. Clearly, all states must have been visited in this optimal teaching length $l$ at least once. So, the corresponding ATSP problem instance must have a solution $\leq l$. But if ATSP has a solution $m<l$, by Claim 1, the METaL problem instance will have a solution $m<l$, thus a contradiction. Hence, the ATSP problem has a solution $l$.
		
		By establishing this reduction, we prove that the METaL problem for a level 3 teacher is at least as hard as ATSP problem which is itself NP-hard.
	\end{proof}
	
	\section{Level 1: Algorithm and Proof}
	\begin{algorithm}[htb]
		\caption{Optimal Level 1 Teaching Algorithm}
		\label{alg:puppet}
		\begin{flushleft}
			\textbf{def Teach($M, L, Q_0, \pi^\dagger$): }
		\end{flushleft}
		\begin{algorithmic}[1]
			\STATE A state $s$ needs to be taught if $Q_0(s,\pi^\dagger(s)) \le \max_{a\neq\pi^\dagger(s)} Q_0(s, a)$.
			Terminate if the MDP has no state to be taught.
			Otherwise arbitrarily order all MDP states that need to be taught as $s^{(0)}, s^{(1)},\cdots, s^{(n)}$ where $0\le n\le S-1$.
			\STATE The teacher provides the state $s_0 \gets s^{(0)}$.
			\FOR{$t = 0,1,\cdots,n$}
			\STATE The agent performs an action according to its current behavior policy $a_{t} \gets \pi_{t}(s_{t})$.
			\STATE The teacher replaces the chosen action with target action $a_{t} \gets $  \pd{s_{t}}.
			\STATE The teacher provides the reward $r_{t}$, and next state $s_{t+1}$
			\STATE \hspace*{.25cm} where $s_{t+1} \gets s^{(\min (t+1, n))}$
			\STATE \hspace*{1.25cm} $r_t : Q_{t+1}(s_{t}, a_t) > \max_{a\neq a_t} Q_{t+1}(s_t, a)$.
			
			\STATE The agent performs an update $Q_{t+1} \gets f(Q_t,e_t)$ using experience $e_t = (s_t, a_t, r_t, s_{t+1})$ 
			
			\ENDFOR
		\end{algorithmic}
	\end{algorithm}

	\begin{proof}[\textbf{Proof of Theorem \ref{thm: lvl 1}}]
		For a level 1 teacher, the worst-case teaching problem instance is the one in which for all states $s \in \S$, the target action $\pi^\dagger(s)$ is not the top action in the $Q_0(s,\cdot)$. In that case, the teacher would need to make the learner visit each state $s$ at least once so that the learner has a chance to learn $\pi^\dagger$ as $s$, 
		i.e. to produce and maintain the eventual condition $Q_T(s,\pi^\dagger(s)) > \max_{a\neq \pi^\dagger(s)}Q_T(s,\cdot)$. 
		Thus, $TDim \geq S$. 
		On the other hand, a level-1 teacher can teach a state in just a single visit to it by replacing the agent chosen action with the target action and rewarding it with a sufficiently high reward (step 8 in the algorithm). 
		Further, at any time step, it can also make the agent transition to an untaught state to teach the target action in that state. Thus, for the worst teaching problem instance, the level-1 teacher can teach the target policy in $S$ steps and hence $TDim=S$.
	\end{proof}
	
	
	\section{Level 2: Algorithm and Proof}

	\begin{algorithm}[H]
		\caption{Optimal Level 2 Teaching Algorithm}
		\label{alg:level2}
		\begin{flushleft}
			\textbf{def Teach($M, L, Q_0, \pi^\dagger$): }
		\end{flushleft}
		\begin{algorithmic}[1]
			\STATE A state $s$ needs to be taught if $Q_0(s,\pi^\dagger(s)) \le \max_{a\neq\pi^\dagger(s)} Q_0(s, a)$.
			Terminate if the MDP has no state to be taught.
			Otherwise arbitrarily order all MDP states that need to be taught as $s^{(0)}, s^{(1)},\cdots, s^{(n)}$ where $0\le n\le S-1$.
			\STATE $t \gets 0, i \gets 0$, the teacher provides initial state $s_0 \gets s^{(0)}$
			\WHILE{$i \le n$}
			\STATE The agent picks a randomized action $a_t \gets \pi_t(s_t)$.
			\IF{$a_t = \pi^\dagger(s_t)$}
			\STATE $s_{t+1} \gets s^{(\min(i+1, n))}$ 
			\STATE $i \gets i+1$ ~~// move on to the next state
			\STATE $r_t : Q_{t+1}(s_{t}, a_t) > \max_{a\neq a_t} Q_{t+1}(s_{t}, a)$ ~~//promote action $\pi^\dagger(s_t)$ to top
			\ELSE
			\IF{ \{a:  $Q_t(s_t, a) \ge Q_t(s_t, \pi^\dagger(s_t))\} = \{a_t, \pi^\dagger(s_t)\} $ }
			\STATE $s_{t+1} \gets s^{(\min(i+1, n))}$ 
			\STATE $i \gets i+1$ ~~// move on to the next state
			\ELSE
			\STATE $s_{t+1} \gets s^{(i)}$ ~~// stay at this state
			\ENDIF
			\STATE $r_t : Q_{t+1}(s_{t}, a_t) < \min_{a \neq a_t}  Q_{t+1}(s_{t}, a)$ ~~// demote action $a_t$ to bottom
			\ENDIF
			\STATE The agent performs an update $Q_{t+1} \gets f(Q_t, e_t)$ with experience $e_t = (s_t, a_t, r_t, s_{t+1})$
			\STATE $t \gets t+1$
			\ENDWHILE
		\end{algorithmic}
	\end{algorithm}
	Remark: Line 10 checks whether $a_t$ is the only no-worse action than $\pi^\dagger(s_t)$: if it is, its demotion also completes teaching at $s_t$.  
	
	%
	%
	
	\begin{proof}[\textbf{Proof of Lemma \ref{thm:expected}}]
		We focus on teaching the target action $\pi^\dagger(s)$ at a particular state $s$. 
		In general let there be $n \in \{1, \ldots, A-1\}$ other actions better than $\pi^\dagger(s)$ in $Q(s,\cdot)$.
		For simplicity, we assume 
		no action is tied with $\pi^\dagger(s)$, namely
		\begin{equation}\label{eqn:q-order}
		Q(s,a_{i_1}) \geq  \cdots \geq Q(s,a_{i_{n}}) > Q(s, a_{i_{n+1}} = \pi^\dagger(s)) >  Q(s, a_{i_{n+2}})\geq \cdots\geq Q(s, a_{i_{A}}).
		\end{equation}	
		Define the upper action set $U := \{ a_{i_1} \cdots a_{i_{n}}\}$ 
		and the lower action set $U := \{ a_{i_{n+2}} \cdots a_{i_{A}}\}$. 
		Define $T(n)$ to be the expected number of visits to $s$ to teach the target action $\pi^\dagger(s)$ at state $s$,  given that initially there are $n$ other actions better than $\pi^\dagger(s)$.
		By ``teach'' we mean move the $n$ actions from $U$ to $L$.
		When the agent visits $s$ it takes a randomized action according to $a_t \gets \pi_t(s)$, which can be any of the $A$ actions.
		We consider three cases:
		\begin{itemize}
			\item[Case 1:] $a_t \in U$, which happens with probability $1-\epsilon+(n-1)\frac{\epsilon}{A-1}$. The teacher provides a reward to demote this action to the bottom of $Q(s,\cdot)$. 
			Therefore, $U$ has one less action after this one teaching step, and recursively needs $T(n-1)$ expected steps in the future.
			
			\item[Case 2:] $a_t = \pi^\dagger(s)$, which happens with probability $\frac{\epsilon}{A-1}$. The teacher provides a reward to promote $a_t$ to the top of $Q(s,\cdot)$ and terminates after this one teaching step (equivalently, $T(0)=0$).
			
			\item[Case 3:] $a_t \in L$, which happens with probability $(A-n-1)\frac{\epsilon}{A-1}$. The teacher can do nothing to promote the target action $\pi^\dagger(s)$ because $a_t$ is
			already below $\pi^\dagger(s)$. Thus, the teacher provides a reward that keeps it that way.
			In the future, it still needs $T(n)$ steps.
		\end{itemize}
		Collecting the 3 cases together we obtain
		\begin{equation}
		T(n) = 1+ \left[ \left(1-\epsilon+(n-1) \frac{\epsilon}{A-1}\right)T(n-1) + \frac{\epsilon}{A-1} T(0) + (A-n-1)\frac{\epsilon}{A-1} T(n) \right].
		\end{equation}
		Rearranging,
		\begin{equation}
		\left(1 - {A-n-1 \over A-1}\epsilon\right)T(n) = 1 + \left(1 - {A-n \over A-1}\epsilon\right) T(n-1).
		\end{equation}
		This can be written as
		\begin{equation}
		\left(1 - {A-1-n \over A-1}\epsilon\right)T(n) = 1 + \left(1 - {A-1-(n-1) \over A-1}\epsilon\right) T(n-1).
		\end{equation}
		This allows us to introduce 
		\begin{equation}
		B(n) := \left(1 - {A-1-n \over A-1}\epsilon\right)T(n)
		\end{equation}
		with the relation
		\begin{equation}
		B(n) = 1 + B(n-1).
		\end{equation}
		Since $T(0)=0$, $B(0)=0$.  Therefore, $B(n)=n$ and
		\begin{equation}
		T(n) = {n \over 1 - {A-1-n \over A-1}\epsilon}.
		\end{equation}
		It is easy to show that the worst case is $n=A-1$, where 
		$T(A-1) = A-1$ regardless of the value of $\epsilon$.  This happens when the target action is originally at the bottom of $Q(s,\cdot)$.
	\end{proof}

	\begin{proof}[\textbf{Proof of Theorem \ref{thm: lvl 2}}]
		We construct a worst-case RL teaching problem instance.  We design $Q_0$ so that for each state $s\in\S$ the target action $\pi^\dagger(s)$ is at the bottom of $Q_0(s, \cdot)$. 
		By Lemma~\ref{thm:expected} the teacher needs to make the agent visits each state $A-1$ times in expectation. Thus a total $S(A-1)$ expected number of steps will be required to teach the target policy to the learner.
	\end{proof}
	
	\section{Level 3 and 4: Algorithm and Proofs}
	\begin{algorithm}[H]
		\caption{The NavTeach Algorithm}
		\label{alg:FAA}
		\begin{flushleft}
			\textbf{def Init($M$): }
		\end{flushleft}
		\begin{algorithmic}[1]
			\STATE $D \leftarrow \infty$. ~~// select the initial state with the shortest tree
			\FOR{$s$ in $\{s \mid \mu_0(s)>0\}$}
			\STATE Construct a minimum depth directed tree $T(s)$ from $s$ to all states in the underlying directed graph of $M$, via breadth first search from $s$. Denote its depth as $D(s)$.
			\IF{$D(s)<D$}
			\STATE depth $D \leftarrow D(s)$; ~ root $s^S \gets s$.
			\ENDIF
			\ENDFOR
			\STATE Let $s^1, \ldots,s^{S-1},s^S$ correspond to a post-order depth-first traversal on the tree $T(s^S)$.
		\end{algorithmic}
		
		\begin{flushleft}
			\textbf{def NavTeach($M, L, Q_0, \pi^\dagger, \delta$): }
		\end{flushleft}
		\begin{algorithmic}[1]
			
			\STATE $t \gets 0, s_t \gets s^S$, ask for randomized agent action $a_t \gets \pi_t(s_t)$
			\FOR{$i = 1, \ldots, S$}
			\STATE // subtask $i$: teach target state $s^i$ with the help of navigation path $p^i$
			\STATE Let $p^i \gets [s_{i_0}=s^S, s_{i_1}, ..., s_{i_d}=s^i]$ be the ancestral path from root $s^S$ to $s^i$ in tree $T(s^S)$
			
			\WHILE {$\argmax_a Q_t(s^i,a) \neq \{\pi^\dagger(s^i)\}$}
			
			\IF {$s_t = s^{i}$}
			\STATE // $s_t=s^i$: the current subtask, establish the target policy.
			\STATE Randomly pick $s_{t+1} \in \{s': P(s' \mid s_t,a_t) >0\}.$
			\IF{$a_t = \pi^\dagger(s_t)$}
			\STATE $r_t \gets \mbox{CarrotStick}(\text{`promote'},a_t,s_t,s_{t+1}, Q_t, \delta)$
			\ELSE 
			\STATE $r_t \gets \mbox{CarrotStick}(\text{`demote'},a_t,s_t,s_{t+1}, Q_t, \delta)$
			\ENDIF
			
			\ELSIF{$s_t \in p^i$}
			\STATE // build navigation if MDP allows
			\IF {$P(p^i.\mbox{next}(s_t) \mid s_t, a_t)>0$}
			\STATE $s_{t+1} \gets p^i.\mbox{next}(s_t)$.
			\STATE $r_t \gets \mbox{CarrotStick}(\text{`promote'},a_t,s_t,s_{t+1},Q_t, \delta)$.
			\ELSE
			\STATE Randomly pick $s_{t+1} \in \{s': P(s' \mid s_t,a_t) >0\}.$
			\STATE $r_t \gets \mbox{CarrotStick}(\text{`demote'},a_t,s_t,s_{t+1},Q_t, \delta)$.
			\ENDIF

			\ELSE
			\STATE // $s_t$ is off subtask $i$ or an already taught state, maintain the $Q(s_t, a_t)$
			\STATE Randomly pick $s_{t+1} \in \{s': P(s' \mid s_t,a_t) >0\}.$
			\STATE $r_t \gets \mbox{CarrotStick}(\text{`maintain'},a_t,s_t,s_{t+1},Q_t, \delta)$.
			\ENDIF
			
			\STATE Give experience $e_t \gets (s_t, a_t, r_t,s_{t+1})$ to the agent.
			\STATE $t \gets t+1$
			\IF{$t \% H = H-1$}
			\STATE $s_{t} \gets s^S$. ~~// episode reset
			\ENDIF
			\STATE Ask for randomized agent action $a_t \gets \pi_t(s_t)$
			\ENDWHILE 
			\ENDFOR		
		\end{algorithmic}
		
		\begin{flushleft}
			\textbf{def CarrotStick($g,a,s,s',Q,\delta$): }
		\end{flushleft}
		\begin{algorithmic}[1]
			\STATE // make $a$ unambiguously the worst action or the best action (with margin  $\delta$)or keep it as it is.
			\IF{$g = $ `promote'}
			\STATE Return $r\geq 0$ such that $Q_{t+1}(s,a) = \underset{b\neq a}{\argmax}\,  Q_{t+1}(s,b)+\delta$ after $Q_{t+1}=f(Q_t, (s, a, r, s'))$.
			\ELSIF{$g = $ `demote'}
			\STATE Return $r \leq 0$ such that $Q_{t+1}(s,a) = \underset{b\neq a}{\argmin}\,  Q_{t+1}(s,b)-\delta$ after $Q_{t+1}=f(Q_t, (s, a, r, s'))$.
			\ELSE
			\STATE Return $r$ such that $Q_{t+1}(s,a) = Q_t(s,a)$ after $Q_{t+1}=f(Q_t, (s, a, r, s'))$.
			\ENDIF
		\end{algorithmic}
		
	\end{algorithm}	
	
	\begin{proof}[\textbf{Proof of Tighter Lower and Upper bound for Level 3 Teacher}] We hereby prove the claimed matching $\Theta\left((S-D)AH(1-\epsilon)^{-D} + H\frac{1-\epsilon}{\epsilon}[(1-\epsilon)^{-D}-1]\right)$ lower and upper bounds for Level 3 Teacher. 
		The key observation is that for an MDP with state space size $S$ and diameter $D$, there must exist $D$ states whose distance to the starting state is $0,1,\ldots,D-1$, respectively. As a result, the total time to travel to these states is at most
		\begin{eqnarray}
		\sum_{d=0}^{D-1} H(\frac{1}{1-\epsilon})^d = H\frac{1-\epsilon}{\epsilon}[(\frac{1}{1-\epsilon})^D-1]
		\end{eqnarray}
		In the lower bound proof of Theorem \ref{thm:l3_lower}, we only count the number of teaching steps required to teach the tail states. Now, if we assume in addition that the neck states also need to be taught, and the target actions are similarly at the bottom of the $Q_0(s,a)$, then it requires precisely an additional $H\frac{1-\epsilon}{\epsilon}[(\frac{1}{1-\epsilon})^D-1]$ steps to teach, which in the end gives a total of
		\begin{equation}
		(S-D-1)(A-1)H(\frac{1}{1-\epsilon})^D + H\frac{1-\epsilon}{\epsilon}[(\frac{1}{1-\epsilon})^D-1]
		\end{equation}
		steps.
		
		In the upper bound proof of Theorem \ref{thm:l3_upper} we upper bound the distance from $s_0$ to any state by $D$. However, based on the observation above, at most $S-D$ states can have distance $D$ from $s_0$, and the rest $D$ states must have distance $0,1,...,D-1$. This allows us to upperbound the total number of teaching steps by
		\begin{equation}
		(2S-1-2D)(A-1)H(\frac{1}{1-\epsilon})^D + H\frac{1-\epsilon}{\epsilon}[(\frac{1}{1-\epsilon})^D-1]
		\end{equation}
		These two bounds matches up to a constant of $2$, and thus gives a matching $\Theta\left((S-D)AH(1-\epsilon)^{-D} + H\frac{1-\epsilon}{\epsilon}[(1-\epsilon)^{-D}-1]\right)$ lower and upper bound.
		Setting $\epsilon = 0$ (requires taking the limit of $\epsilon\rightarrow0$) induces Corollary \ref{thm: cor}.
	\end{proof}

	\begin{figure}[ht]
		\begin{center}
			\includegraphics[width=0.75\textwidth]{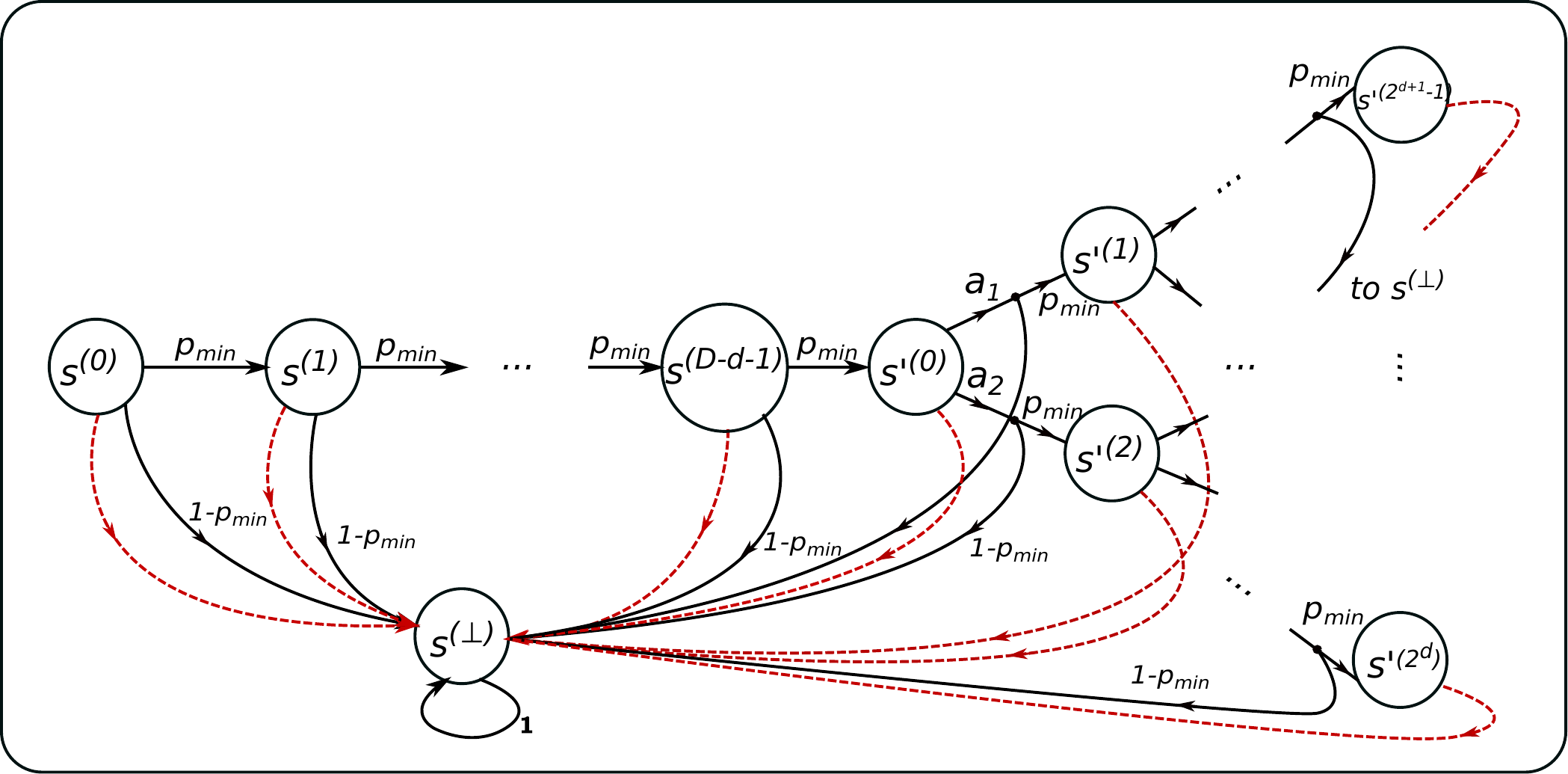}
			\caption{The ``peacock tree" MDP}
			\label{fig:peacock_tree}
		\end{center}
	\end{figure}
	
	\begin{proof}[\textbf{Proof of Theorem \ref{thm: l4_lower}}]
		We construct a hard level 4 teaching problem instance, very similar to ``peacock MDP'' and call it ``peacock tree MDP''. We then show that this MDP admits the given lower bound. The ``peacock tree MDP'' has a linear chain of length $D-d-1$(the ``neck'') and a $d$ depth binary tree(the ``tail'') attached to the end of the neck. For a given $(S,D)$, we can always find $d$ such that $2^{d}+(D-d+1) \leq S \leq 2^{d+1}+(D-d)$. Note that the depth of this MDP is $D$. To simplify the analysis of the proof, from now on, we will assume that the binary tree is complete and full, i.e. $S = 2^{d+1}+(D-d)$.
		
		As in the case of ``peacock MDP'', every state has $A$ actions. The action $a_1$ in the chain transits to next state with probability $p_{min}$ and to the absorbing state $s^{(\perp)}$ with probability $1-p_{min}$. The action $a_1$ in the non-leaf states of the binary tree transits to its top child with probability $p_{min}$ and to $s^{(\perp)}$ with probability $1-p_{min}$, the action $a_2$ there transits to the bottom child with probability $p_{min}$ and to $s^{(\perp)}$ with probability $1-p_{min}$. All other $A-1$ actions in the non-leaf states and the chain states lead to $s^{(\perp)}$ with probability 1. Further, all $A$ actions in the leaf states lead to $s^{(\perp)}$ with probability 1. The target policy is to select $a_1$ at every state. We consider an initial $Q_0$ which favors the target policy at all non-leaf and chain states. For all the leaf states $s$ the target action $a_1$ is $\arg \min_{a} Q_0(s,a)$, namely at the bottom and needs to be taught.
		
		For the lower bound analysis, we consider teaching each leaf state when the traversal path to it is already optimal (Note that in reality, the path has to be taught for each leaf states, but that will eventually add to the lower bound, so we omit it for this analysis). 
		For a leaf state $s$, there exists a path from the root to it.  This requires the teacher to provide the correct transition to the next state along the path, and the learner to choose actions $a_{1}$ all along the chain and then a combination of $a_1$ and $a_2$ actions to reach that leaf $s$. 
		Given that the traversal path to the leaf is already optimal, a successful episode consists of the learner choosing the greedy action at each step and the teacher transitioning the learner to the correct next state on the path to the leaf, which happens with a probability of $(p_{min}(1-\epsilon))^D$. Thus, the expected number of episodes required to make the learner visit the leaf and teach it once there is $(\frac{1}{p_{min}(1-\epsilon)})^D$. Note that in a successful episode, the learner takes $D$ steps to reach the leaf and the rest of the steps in that episode is wasted, thus accounting for a total of $H$ steps. Similarly, any failed episode wastes a total of $H$ steps. Hence, the expected number of steps required to visit and teach a leaf state once is at least $H(\frac{1}{p_{min}(1-\epsilon)})^D$. 
		The teacher has to make the learner visit all $2^{d}$ leaf states $A-1$ times in expectation (since by our construction, the target action of each leaf is at the bottom of the Q-value ordering). Collectively, this would require at least  $2^{d}(A-1)H(\frac{1}{(p_{min}(1-\epsilon)})^D$ steps. We note that, $S = 2^{d+1}+(D-d) \leq 2^{d+1}+D = 2\cdot2^{d} + D \implies 2^{d} \geq \frac{1}{2}(S-D)$. Thus, the expected number of steps to teach the target policy is $\geq \frac{1}{2}(S-D)(A-1)H(\frac{1}{p_{min}(1-\epsilon))^D}) \implies TDim \geq \Omega ((S-D)AH(\frac{1}{p_{min}(1-\epsilon))^D}))$.
	\end{proof}

	
	\begin{proof}[\textbf{Proof of Theorem \ref{thm: l4_upper}}]
		The proof follows similarly to the upper bound proof for the teaching dimension of a level 3 teacher and uses NavTeach algorithm \algref{alg:FAA}. For a given MDP, the teacher first creates a breadth-first tree and then starts teaching the states in a post-order depth-first traversal. Note that the breadth-first tree is still constructed using the transition edges that are supported by the underlying MDP. A level 4 teacher, while transitioning out from a particular state, can only choose a desired transition-edge with a probability $\geq p_{min}$. Thus, the probability that the teacher can make the learner transit from one state to another using a greedy action chosen by the learner is at least $p_{min}(1-\epsilon)$. 
		
		The teaching goal is broken into $S$ subtasks, one for each state. The sub-task for a state further consists of teaching a navigation path to reach that state and then teaching the target action in that state. Because of the post-order depth-first teaching strategy, a large part of the navigation path is shared between two subtasks. Also, this strategy requires a navigation action at each non-leaf state to be taught just once. We further note that in depth-first teaching strategy, a navigation action from a parent state $s_i$ to a child state $s_j$ is taught only after a navigation path to the parent $s_i$ is laid. Similarly, the target action at a state $s_j$ is taught only after a navigation path to it is laid. Thus, the expected number of steps required to reach a state at depth $i$ and teach once there is at most $(\frac{1}{p_{min}(1-\epsilon)})^i$. For a simpler analysis, we assume that once the agent falls off the path leading to the target state, the remaining steps in that episode are wasted. Similarly, once an agent reaches a target state and is taught by the teacher, the remaining episode steps are wasted. 
		Thus, the expected number of steps required to visit a state at depth $i$ and teach the navigation action there is $(A-1)H(\frac{1}{p_{min}(1-\epsilon)})^i \leq (A-1)H(\frac{1}{p_{min}(1-\epsilon)})^D$. Noting the fact that there are at most $S-1$ non-leaf states and the teacher needs to teach the navigation action at each of them exactly once, the expected number of steps required to teach all the navigation actions is at most
		\begin{equation}\label{eqn:navigate_action}
		(S-1)(A-1)H\Big(\frac{1}{p_{min}(1-\epsilon)}\Big)^D.
		\end{equation}  
		Similarly, the expected number of steps required to visit a state at depth $i$ and teach the target action there is $(A-1)H(\frac{1}{p_{min}(1-\epsilon)})^i \leq (A-1)H(\frac{1}{p_{min}(1-\epsilon)})^D$. Adding it up, the expected number of steps required to teach the target action at all states is at most 
		
		\begin{equation}\label{eqn:target_action}
		S(A-1)H\Big(\frac{1}{p_{min}(1-\epsilon)}\Big)^D.
		\end{equation}
		Combining \ref{eqn:navigate_action} and \ref{eqn:target_action}, we conclude that the expected number of steps required to teach the target policy using \algref{alg:FAA} is at most 
		\begin{equation}
		(2S-1)(A-1)H(\frac{1}{p_{min}(1-\epsilon)})^D \implies TDim \leq  O(SAH\Big(\frac{1}{p_{min}(1-\epsilon)}\Big)^D).
		\end{equation} 
	\end{proof}
	
	\paragraph{Remark:} A more careful analysis that leads to a tight lower and upper bound is also possible for the level 4 teacher, but the calculation and the eventual bound one gets become much more complicated, and thus we defer it to future works.
	
	\section{Generalization to SARSA}
	\begin{algorithm}[ht!]
		\caption{Machine Teaching Protocol for SARSA}\label{alg:sarsa_protocol}
		\begin{flushleft}
			\textbf{Entities:} MDP environment, learning agent with initial Q-table $Q_0$, teacher.\\
		\end{flushleft}
		\begin{algorithmic}[1]
			\STATE MDP draws $s_0 \sim \mu_0$ after each episode reset. But the teacher \textbf{may} override $s_0$.
			\FOR{$t = 0,...,H-1$} 
			\STATE The agent picks an action $a_{t} = \pi_{t}(s_{t})$ with its current behavior policy $\pi_{t}$. But the teacher \textbf{may} override $a_t$ with a teacher-chosen action.
			\IF{$t=0$}
			\STATE The agent updates  $Q_{t+1} = Q_t$.
			\ELSE
			\STATE The agent updates $Q_{t+1} = f(Q_{t},e_{t})$ from experience $e_{t} = (s_{t-1}, a_{t-1}, r_{t-1}, s_{t}, a_{t})$.
			\ENDIF
			
			\STATE The MDP evolves from $(s_{t},a_{t})$ to produce immediate reward $r_{t}$ and the next state $s_{t+1}$. But the teacher \textbf{may} override $r_{t}$ or move the system to a different state $s_{t+1}$.
			\ENDFOR
		\end{algorithmic}
	\end{algorithm}
	SARSA is different from standard Q-learning in that its update is delayed by one step. In time step $t$, the agent is updating the $(s_{t-1},a_{t-1})$ entry of the Q table, using experience $e_{t} = (s_{t-1}, a_{t-1}, r_{t-1}, s_{t}, a_{t})$. This delayed update makes the student learn slowly. In particular, we show that it can take twice as many visits to a state to enforce the target action compared to Q-learning.
	\begin{lemma}\label{thm:expected_sarsa}
		For a Level 2 Teacher, any SARSA learner, and an MDP family $\mathcal M$ with action space size $A$, it takes at most $2A-2$ visits in expectation to a state $s$ to teach the desired action $\pi^\dagger(s)$ on $s$.
	\end{lemma}
	\begin{proof}[Proof Sketch:]
		The key in proving Lemma \ref{thm:expected_sarsa} is to see that if the agent visits the same state two times in a row, then the lesson provided by the teacher during the first visit has not been absorbed by the learner, and as a result, during the second visit, the learner will still prefer the same (undesirable) action. This, in the worst case ($\epsilon=0$), will be a completely wasted time step, which implies that the total number of visits required will double compared to Q-learning, giving us $2A-2$.
	\end{proof}
	The wasted time step in Lemma \ref{thm:expected_sarsa} will only occur when the agent visits one state twice in a roll. This can be avoided in Level 1 and 2 teachers as long as $S\geq 2$. Therefore, the teaching dimension for level 1 and 2 teachers will only increase by $1$ due to the delayed update of the learner. For Level 3 and Level 4 teacher, the new Lemma \ref{thm:expected_sarsa} only results in at most 2 times increase in the teaching dimension, which does not change the order of our results. Therefore, Level 3 and Level 4 results still hold for SARSA agents.
	
\end{document}